\documentclass{article}
\usepackage{amsfonts}
\usepackage{amsmath}

\newtheorem{theorem}{Theorem}

\newtheorem{conjecture}[theorem]{Conjecture}
\newtheorem{corollary}[theorem]{Corollary}

\newtheorem{lemma}[theorem]{Lemma}

\newtheorem{proposition}[theorem]{Proposition}

\newenvironment{proof}[1][Proof]{\noindent\textbf{#1.} }{\ \rule{0.5em}{0.5em}}
\input{tcilatex}
\begin{document}

\title{A vector-contraction inequality for Rademacher complexities}
\author{Andreas Maurer \\
Adalbertstr. 55 \ D-80799 M\"{u}nchen\newline
\\
Germany\\
am@andreas-maurer.eu}
\maketitle

\begin{abstract}
The contraction inequality for Rademacher averages is extended to Lipschitz
functions with vector-valued domains, and it is also shown that in the
bounding expression the Rademacher variables can be replaced by arbitrary
iid symmetric and sub-gaussian variables. Example applications are given for
multi-category learning, K-means clustering and learning-to-learn.
\end{abstract}

\section{Introduction}

The method of Rademacher complexities has become a popular tool to prove
generalization in learning theory. One has the following result \cite%
{Bartlett 2002}, which gives a bound on the estimation error, uniform over a
loss class $\tciFourier $. 

\begin{theorem}
\label{Theorem Rademacher bound}Let $\mathcal{X}$ be any set, $\tciFourier $
a class of functions $f:\mathcal{X\rightarrow }\left[ 0,1\right] $ and let $%
X,X_{1},...,X_{n}$ be iid random variables with values in $\mathcal{X}$.
Then for $\delta >0$ with probability at least $1-\delta $ in , $\mathbf{X}%
=\left( X_{1},...,X_{n}\right) $ we have for every $f\in \tciFourier $ that%
\begin{equation*}
\mathbb{E}f\left( X\right) \leq \frac{1}{n}\sum f\left( X_{i}\right) +\frac{2%
}{n}\mathbb{E}\left[ \sup_{f\in \tciFourier }\sum_{i=1}^{n}\epsilon
_{i}f\left( X_{i}\right) |\mathbf{X}\right] +\sqrt{\frac{9\ln 2/\delta }{2n}}%
.
\end{equation*}
\end{theorem}

Here the $\epsilon _{1},...,\epsilon _{n}$ are (and will be throughout this
paper) independent Rademacher variables, uniformly distributed on $\left\{
-1,1\right\} $. For any class $\tciFourier $ of real, not necessarily $\left[
0,1\right] $-valued, functions defined on $\mathcal{X}$, and any vector $%
\mathbf{x}=\left( x_{1},...,x_{n}\right) \in \mathcal{X}^{n}$, the quantity 
\begin{equation*}
\mathbb{E}\sup_{f\in \tciFourier }\sum_{i=1}^{n}\epsilon _{i}f\left(
x_{i}\right) 
\end{equation*}%
is called the Rademacher complexity of the class $\tciFourier $ on the
sample $x=\left( x_{1},...,x_{n}\right) \in \mathcal{X}^{n}$. Here we omit
the customary factor $2/n$, as this will simplify most of our statements
below.\bigskip 

Most applications of the method at some point or another use the so-called
contraction inequality. For functions $h_{i}:%
\mathbb{R}
\rightarrow 
\mathbb{R}
$ with Lipschitz constant $L$, the scalar contraction inequality states that%
\begin{equation*}
\mathbb{E}\sup_{f\in \tciFourier }\sum_{i=1}^{n}\epsilon _{i}h_{i}\left(
f\left( x_{i}\right) \right) \leq L\mathbb{E}\sup_{f\in \tciFourier
}\sum_{i=1}^{n}\epsilon _{i}f\left( x_{i}\right) .
\end{equation*}%
\bigskip 

There are situations when it is desirable to extend this result to the case,
when the class $\tciFourier $ consists of vector-valued functions and the
loss functions are Lipschitz functions defined on a more than
one-dimensional space. Such occurs for example in the analysis of
multi-class learning, $K$-means clustering or learning-to-learn. At present
one has dealt with these problems by passing to Gaussian averages and using
Slepian's inequality (see e.g. Theorem 14 in \cite{Bartlett 2002}). This is
sufficient for many applications, but there are two drawbacks: 1. the proof
relies on a highly nontrivial result (Slepian's inequality) and 2. while
Rademacher complexities are tightly bounded in terms of Gaussian
complexities, it is well known (\cite{Ledoux Talagrand 1991}, \cite%
{Boucheron 2013}) that bounding the latter in terms of the former incurs a
factor logarithmic in the number of variables, potentially resulting in an
unnecessary weakening of the results (see e.g. \cite{Kloft 2015}).

In this paper we will prove the vector contraction inequality%
\begin{equation}
\mathbb{E}\sup_{f\in \tciFourier }\sum_{i=1}^{n}\epsilon _{i}h_{i}\left(
f\left( x_{i}\right) \right) \leq \sqrt{2}L\mathbb{E}\sup_{f\in \tciFourier
}\sum_{i=1}^{n}\sum_{k=1}^{K}\epsilon _{ik}f_{k}\left( x_{i}\right) ,
\label{Vector Rademacher Contraction Inequality}
\end{equation}%
where the members of $\tciFourier $ take values in $%
\mathbb{R}
^{K}$ with component functions $f_{k}\left( .\right) $, the $h_{i}$ are $L$%
-Lipschitz functions from $%
\mathbb{R}
^{K}$ to $%
\mathbb{R}
$, and the $\epsilon _{ik}$ are an $n\times K$ matrix of independent
Rademacher variables. It is also shown that the $\epsilon _{ik}$ on the
right hand side of (\ref{Vector Rademacher Contraction Inequality}) can be
replaced by arbitrary iid random variables as long as they are symmetric and
sub-gaussian, and $\sqrt{2}$ is replaced by a suitably chosen constant.
Furthermore the result extends to infinite dimensions in the sense that $%
\mathbb{R}
^{K}$ can be replaced by the Hilbert space $\ell _{2}$. The proof given is
self-contained and independent of Slepian's inequality.\bigskip 

We illustrate applications of this inequality by showing that it applies to
loss function in a variety of relevant cases. In Section \ref{Section loss
functions} we discuss multi-class learning, $K$-means clustering and
learning-to-learn. We also give some indications of how the vector-valued
complexity on the right hand side of (\ref{Vector Rademacher Contraction
Inequality}) may be bounded. An example pertaining to the truly infinite
dimensional case is given, generalizing some bounds for least-squares
regression with operator valued kernels (\cite{Massi 2005}, \cite{Caponnetto
2007}) to more general loss-functions.\bigskip 

The inequality (\ref{Vector Rademacher Contraction Inequality}) is perhaps
not the most natural form of a vector-contraction inequality, and, since the
right hand side is sometimes difficult to bound, one is led to look for
alternatives. An attractive conjecture might be the following.

\begin{conjecture}
\label{Conjecture False}\textit{Let }$\mathcal{X}$\textit{\ be any set, }$%
n\in 
\mathbb{N}
$\textit{, }$\left( x_{1},...,x_{n}\right) \in X^{n}$\textit{, let }$%
\tciFourier $\textit{\ be a class of functions }$f:\mathcal{X}\rightarrow
\ell _{2}$\textit{\ and let }$h:\ell _{2}\rightarrow 
\mathbb{R}
$\textit{\ have Lipschitz norm }$L$\textit{. Then}%
\begin{equation*}
\mathbb{E}\sup_{f\in \tciFourier }\sum_{i}\epsilon _{i}h\left( f\left(
x_{i}\right) \right) \leq KL\mathbb{E}\sup_{f\in \tciFourier }\left\Vert
\sum_{i}\epsilon _{i}f\left( x_{i}\right) \right\Vert ,
\end{equation*}%
\textit{where }$K$\textit{\ is some universal constant.}
\end{conjecture}

This conjecture is false and will be disproved in the sequel.\bigskip

A version of the scalar contraction inequality occurs in \cite{Ledoux
Talagrand 1991}, Theorem 4.12. There the absolute value of the Rademacher
sum is used, and a necessary factor of two appears on right hand side. With
the work of Koltchinski and Panchenko \cite{Kol} and Bartlett and Mendelson 
\cite{Bartlett 2002} Rademacher averages became attractive to the machine
learning community, there was an increased interest in contraction
inqualities and it was realized that the absolute value was unnecessary for
most of the new applications. Meir and Zhang \cite{Meir} gave a nice and
simple proof of the scalar contraction inequality as stated above. Our proof
of (\ref{Vector Rademacher Contraction Inequality}) is an extension of their
method. 

\section{The vector-contraction inequality}

All random variables are assumed to be defined on some probability space $%
\left( \Omega ,\Sigma \right) $. The space $L_{p}\left( \Omega ,\Sigma
\right) $ is abbreviated $L_{p}$. We use $\ell _{2}$ to denote the Hilbert
space of square summable sequences of real numbers. The norm on $\ell _{2}$
and the Euclidean norm on $%
\mathbb{R}
^{K}$ are denoted with $\left\Vert .\right\Vert $.

A real random variable $X$ is called \textit{symmetric} if $-X$ and $X$ are
identically distributed. It is called \textit{sub-gaussian} if there exists
a constant $b=b\left( X\right) $ such that for every $\lambda \in 
\mathbb{R}
$%
\begin{equation*}
\mathbb{E}e^{\lambda X}\leq e^{\frac{\lambda ^{2}b^{2}}{2}}.
\end{equation*}%
We call $b$ the \textit{sub-gaussian parameter} of $X$. Rademacher and
standard normal variables are symmetric and sub-gaussian. \bigskip 

The following is the main result of this paper.

\begin{theorem}
\label{Theorem Main}Let $X$ be nontrivial, symmetric and subgaussian. Then
there exists a constant $C<\infty $, depending only on the distribution of $X
$, such that for any countable set $\mathcal{S}$ and functions $\psi _{i}:%
\mathcal{S}\rightarrow 
\mathbb{R}
$, $\phi _{i}:\mathcal{S}\rightarrow \ell _{2}$, $1\leq i\leq n$ satisfying%
\begin{equation*}
\forall s,s^{\prime }\in \mathcal{S}\text{, }\psi _{i}\left( s\right) -\psi
\left( s_{i}^{\prime }\right) \leq \left\Vert \phi _{i}\left( s\right) -\phi
_{i}\left( s^{\prime }\right) \right\Vert 
\end{equation*}%
we have.%
\begin{equation*}
\mathbb{E}\sup_{s\in \mathcal{S}}\sum_{i}\epsilon _{i}\psi _{i}\left(
s\right) \leq C~\mathbb{E}\sup_{s\in \mathcal{S}}\sum_{i,k}X_{ik}\phi
_{i}\left( s\right) _{k},
\end{equation*}%
where the $X_{ik}$ are independent copies of $X$ for $1\leq i\leq n$ and $%
1\leq k\leq \infty $, and $\phi _{i}\left( s\right) _{k}$ is the $k$-th
coordinate of $\phi _{i}\left( s\right) $.

If $X$ is a Rademacher variable we may choose $C=\sqrt{2}$, if $X$ is
standard normal $C=\sqrt{\pi /2}$.\bigskip 
\end{theorem}

For applications in learning theory we can at once substitute a Rademacher
variable for $X$ and $\sqrt{2}$ for $C$. For $\mathcal{S}$ we take a class $%
\tciFourier $ of vector valued functions $f:\mathcal{X}\rightarrow \ell _{2}$
, for the $\phi _{i}$ the evaluation functionals on a sample $\left(
x_{1},...,x_{n}\right) $, so that $\phi _{i}\left( f\right) =f\left(
x_{i}\right) $ and for $\psi _{i}$ we take the evaluation functionals
composed with a Lipschitz loss function $h:\ell _{2}\rightarrow 
\mathbb{R}
$ of Lipschitz norm $L$. We obtain the

\begin{corollary}
\label{Corollary Rademacher}Let $\mathcal{X}$ be any set, $\left(
x_{1},...,x_{n}\right) \in \mathcal{X}^{n}$, let $\tciFourier $ be a class
of functions $f:\mathcal{X}\rightarrow \ell _{2}$ and let $h_{i}:\ell
_{2}\rightarrow 
\mathbb{R}
$ have Lipschitz norm $L$. Then%
\begin{equation*}
\mathbb{E}\sup_{f\in \tciFourier }\sum_{i}\epsilon _{i}h_{i}\left( f\left(
x_{i}\right) \right) \leq \sqrt{2}L\mathbb{E}\sup_{f\in \tciFourier
}\sum_{i,k}\epsilon _{ik}f_{k}\left( x_{i}\right) ,
\end{equation*}%
where $\epsilon _{ik}$ is an independent doubly indexed Rademacher sequence
and $f_{k}\left( x_{i}\right) $ is the $k$-th component of $f\left(
x_{i}\right) $.\bigskip 
\end{corollary}

Clearly finite dimensional versions are obtained by restricting to the
subspace spanned by the first $K$ coordinate functions in $\ell _{2}$.

\section{Examples of loss functions\label{Section loss functions}}

We give some examples of seemingly complicated loss functions to which
Theorem \ref{Theorem Main} and Corollary \ref{Corollary Rademacher} can be
applied. These examples are not exhaustive, in fact it seems that many
applications of Slepian's inequality in the machine learning literature can
be circumvented by Theorem \ref{Theorem Main}.

\subsection{Multi-class classification}

Consider the problem of assigning to inputs taken from a space $\mathcal{X}$
a label corresponding to one of $K$ classes. We are given a labelled iid
sample $\mathbf{z}=\left( \left( x_{1},y_{1}\right) ,...,\left(
x_{n},y_{n}\right) \right) $ drawn from some unknown distribution on $%
\mathcal{X\times }\left\{ 1,...,K\right\} $, where the points $x_{i}$ are
inputs $x_{i}\in \mathcal{X}$ and the $y_{i}$ are corresponding labels, $%
y_{i}\in \left\{ 1,...,K\right\} $. Many approaches assume that there is
class $\tciFourier $ of vector valued functions $f:\mathcal{X\rightarrow 
\mathbb{R}
}^{K^{\prime }}$, where $K^{\prime }=o\left( K\right) $ (typically $%
K^{\prime }=K$, for 1-versus-all classification, or $K^{\prime }=K-1$ for
simplex coding \cite{Rosasco 2012}), a classification rule $c:\mathcal{%
\mathbb{R}
}^{K^{\prime }}\rightarrow \left\{ 1,...,K\right\} $, and for each label $%
k\in \left\{ 1,...,c\right\} $ a loss function $\ell _{k}:%
\mathbb{R}
^{K}\rightarrow 
\mathbb{R}
_{+}$. The loss function $\ell _{k}$ is designed so as to upper bound, or
approximate the indicator function of the set $\left\{ z\in 
\mathbb{R}
^{K^{\prime }}:c\left( z\right) \neq k\right\} $ (see \cite{Crammer 2002}).
In most cases the loss functions are Lipschitz on $%
\mathbb{R}
^{K}$ relative to the euclidean norm, with some Lipschitz constant $L$. The
empirical error incurred by a function $f\in \tciFourier $ is 
\begin{equation*}
\frac{1}{n}\sum_{i}\ell _{y_{i}}\left( f\left( x_{i}\right) \right) .
\end{equation*}%
The Rademacher complexity, which would lead to the uniform bound on the
estimation error, is%
\begin{equation*}
\mathbb{E}\sup_{f\in \tciFourier }\sum_{i}\epsilon _{i}\ell _{y_{i}}\left(
f\left( x_{i}\right) \right) .
\end{equation*}%
Using Corollary \ref{Corollary Rademacher} with $h_{i}=\ell _{y_{i}}$ we can
immediately eliminate the loss functions $\ell _{y_{i}}$ 
\begin{equation*}
\mathbb{E}\sup_{f\in \tciFourier }\sum_{i}\epsilon _{i}\ell _{y_{i}}\left(
f\left( x_{i}\right) \right) \leq \sqrt{2}L\mathbb{E}\sup_{f\in \tciFourier
}\sum_{i,k}\epsilon _{ik}f_{k}\left( x_{i}\right) .
\end{equation*}%
How we proceed to further bound this now depends on the nature of the
vector-valued class $\tciFourier $. Some techniques to bound the Rademacher
complexity of vector valued classes are sketched in Section \ref{Section
Bounding} below.

\subsection{K-means clustering}

Let $H$ be a Hilbert space and $\mathbf{x}=\left( x_{1},...,x_{n}\right) $ a
sample of points in the unit ball $B_{1}$ of $H$. The algorithm seeks
centers $c=\left( c_{1},...,c_{K}\right) \in \mathcal{S}=B_{1}^{K}$ to
represent the sample.%
\begin{equation*}
c^{\ast }=\arg \min_{\left( c_{1},...,c_{K}\right) \in \mathcal{S}}\frac{1}{n%
}\sum_{i=1}^{n}\min_{k=1}^{K}\left\Vert x_{i}-c_{k}\right\Vert ^{2}.
\end{equation*}%
The corresponding Rademacher average to bound the estimation error is%
\begin{equation*}
R\left( \mathcal{S},\mathbf{x}\right) =\mathbb{E}\sup_{c\in \mathcal{S}%
}\sum_{i}\epsilon _{i}\min_{k=1}^{K}\left\Vert x_{i}-c_{k}\right\Vert ^{2}=%
\mathbb{E}\sup_{c\in \mathcal{S}}\sum_{i}\epsilon _{i}\psi _{i}\left(
c\right) ,
\end{equation*}%
where we define $\psi _{i}\left( c\right) =\min_{k}\left\Vert
x_{i}-c_{k}\right\Vert ^{2}$ in preparation of an application of Theorem \ref%
{Theorem Main}. The next step is to search for an appropriate Lipschitz
property of the $\psi _{i}$. We have, for $c,c^{\prime }\in \mathcal{S}$,%
\begin{eqnarray*}
\psi _{i}\left( c\right) -\psi _{i}\left( c^{\prime }\right) 
&=&\min_{k}\left\Vert x_{i}-c_{k}\right\Vert ^{2}-\min_{k}\left\Vert
x_{i}-c_{k}^{\prime }\right\Vert ^{2} \\
&\leq &\max_{k}\left\Vert x_{i}-c_{k}\right\Vert ^{2}-\left\Vert
x_{i}-c_{k}^{\prime }\right\Vert ^{2} \\
&\leq &\left( \sum_{k}\left( \left\Vert x_{i}-c_{k}\right\Vert
^{2}-\left\Vert x_{i}-c_{k}^{\prime }\right\Vert ^{2}\right) ^{2}\right)
^{1/2} \\
&=&\left\Vert \phi _{i}\left( c\right) -\phi _{i}\left( c^{\prime }\right)
\right\Vert .
\end{eqnarray*}%
Where we defined $\phi _{i}:\mathcal{S\rightarrow 
\mathbb{R}
}^{K}$ by $\phi _{i}\left( c\right) =\left( \left\Vert
x_{i}-c_{1}\right\Vert ^{2},...,\left\Vert x_{i}-c_{K}\right\Vert
^{2}\right) $. We can now apply Theorem \ref{Theorem Main} with $L=1$ and
obtain 
\begin{eqnarray*}
2^{-1/2}R\left( \mathcal{S},\mathbf{x}\right)  &\leq &\mathbb{E}\sup_{c\in 
\mathcal{S}}\sum_{ik}\epsilon _{ik}\left\Vert x_{i}-c_{k}\right\Vert ^{2} \\
&\leq &2\mathbb{E}\sup_{c\in \mathcal{S}}\sum_{ik}\epsilon _{ik}\left\langle
x_{i},c_{k}\right\rangle +\mathbb{E}\sup_{c\in \mathcal{S}}\sum_{ik}\epsilon
_{ik}\left\Vert c_{k}\right\Vert ^{2} \\
&\leq &K\left( 2\mathbb{E}\left\Vert \sum_{i}\epsilon _{i}x_{i}\right\Vert +%
\mathbb{E}\left\vert \sum_{i}\epsilon _{i}\right\vert \right)  \\
&\leq &3K\sqrt{n}.
\end{eqnarray*}%
Dividing by $n$ we obtain generalization bounds as in \cite{Lugosi 2008} or 
\cite{Maurer 2010}. In this simple case it was very easy to explicitely
bound the complexity of the vector-valued class. 

\subsection{Learning to learn or meta-learning}

With input space $\mathcal{X}$ suppose we have a class $\mathcal{H}$ of
feature maps $h:\mathcal{X}\rightarrow \mathcal{Y\subseteq 
\mathbb{R}
}^{K}$ and a loss class $\tciFourier $ of functions $f:\mathcal{Y}%
\rightarrow \left[ 0,1\right] $. The loss class could be used for
classification or function estimation or also in some unsupervised setting.
We assume that every function $f\in \tciFourier $ is Lipschitz with
Lipschitz constant $L$ and that $\tciFourier $ is small enough for good
generalization in the sense that for some $B<\infty $%
\begin{equation}
\mathbb{E}_{Y_{1},...,Y_{n}}\left[ \sup_{f\in \tciFourier }\mathbb{E}%
_{Y}f\left( Y\right) -\frac{1}{n}\sum_{i=1}^{n}f\left( Y_{i}\right) \right]
\leq \frac{B}{\sqrt{n}}  \label{LTL-bound 1}
\end{equation}%
for any $\mathcal{Y}$-valued random variable $Y$ and iid copies $%
Y_{1},...,Y_{n}$. Such conditions might be established using standard
techniques, for example also Rademacher complexities.\bigskip 

We now want to learn a feature map $h\in \mathcal{H}$, such that the
algorithm using empirical risk minimization (ERM) with the function class $%
\tciFourier \circ h=\left\{ x\mapsto f\left( h\left( x\right) \right) :f\in
\tciFourier \right\} $ gives good results on future, yet unseen, tasks. Of
course this depends on the tasks in question, and a good feature map $h$ can
only be chosen on the basis of some kind of experience made with these tasks.

To formalize this Baxter \cite{Baxter 2000} has introduced the notion of an 
\textit{environment }$\eta $, which is a distribution on the set of tasks,
where each task $t$ is characterized by some distribution $\mu _{t}$ (e.g.
on inputs and outputs). For each task $t\sim \eta $ we can then also draw an
iid training sample $\mathbf{x}^{t}=\left( x_{1}^{t},...,x_{n}^{t}\right)
\sim \mu _{t}^{n}$. In this way the environment also induces a distribution
on the set of training samples. Now we can make our problem more precise:

Suppose we have $T$ tasks and corresponding training samples $\mathbf{\bar{x}%
}=\left( \mathbf{x}^{1},...,\mathbf{x}^{T}\right) $ drawn iid from the
environment $\eta $. For $h\in \mathcal{H}$ let 
\begin{equation*}
\psi _{t}\left( h\right) =\min_{f\in \tciFourier }\frac{1}{n}%
\sum_{i=1}^{n}f\left( h\left( x_{i}^{t}\right) \right) 
\end{equation*}%
be the training error obtained by the use of the feature map $h$. We propose
to use the feature map%
\begin{equation*}
\hat{h}=\arg \min_{h\in \mathcal{H}}\frac{1}{T}\sum_{t=1}^{T}\psi _{t}\left(
h\right) .
\end{equation*}%
To give a performance guarantee for ERM using $\tciFourier \circ \hat{h}$,
we now seek to bound the expected \textit{training error} $\mathbb{E}_{t\sim
\eta }\left[ \psi _{t}\left( h\right) \right] $ for a new task drawn from
the environment (with corresponding training sample), in terms of the
average of the observed $\psi _{t}\left( h\right) $, uniformly over the set
of feature maps $h\in \mathcal{H}$. Observe that, given the bound on (\ref%
{LTL-bound 1}) such a bound will also give a bound on the expected true
error when using $\hat{h}$ on new tasks in the environment $\eta $, a
meta-generalization bound, so to speak (for more details on this type of
argument see \cite{Maurer 2009} or \cite{Maurer 2015}).

The Rademacher average in question is 
\begin{equation*}
R\left( \mathcal{H},\mathbf{\bar{x}}\right) =\mathbb{E}\sup_{h\in \mathcal{H}%
}\sum_{t=1}^{T}\epsilon _{t}\psi _{t}\left( h\right) .
\end{equation*}%
To apply Theorem \ref{Theorem Main} we look for a Lipschitz property of the $%
\psi _{t}$. For $h\in \mathcal{H}$ define $\phi _{t}\left( h\right) \in 
\mathbb{R}
^{Kn}$ by $\left[ \phi _{t}\left( h\right) \right] _{k,i}=h_{k}\left(
x_{i}^{t}\right) $. Then for $h,h^{\prime }\in \mathcal{H}$ 
\begin{eqnarray*}
\psi _{t}\left( h\right) -\psi _{t}\left( h^{\prime }\right)  &\leq
&\max_{f\in \tciFourier }\frac{1}{n}\sum_{i=1}^{n}f\left( h\left(
x_{i}\right) \right) -f\left( h\left( x_{i}\right) \right)  \\
&\leq &\frac{L}{n}\sum_{i=1}^{n}\left\Vert h\left( x_{i}\right) -h^{\prime
}\left( x_{i}\right) \right\Vert \leq \frac{L}{\sqrt{n}}\left\Vert \phi
_{t}\left( h\right) -\phi _{t}\left( h^{\prime }\right) \right\Vert ,
\end{eqnarray*}%
where the first inequality comes from the Lipschitz property of the
functions in $\tciFourier $ and the second from Jensen's inequality. From
Theorem \ref{Theorem Main} we conclude that%
\begin{equation*}
R\left( \mathcal{H},\mathbf{\bar{x}}\right) \leq \frac{L}{\sqrt{n}}\mathbb{E}%
\sup_{h\in \mathcal{H}}\sum_{tki}\epsilon _{tki}h_{k}\left( x_{i}^{t}\right)
.
\end{equation*}%
How to proceed depends on the nature of the feature maps in $\mathcal{H}$.
Examples are given in \cite{Maurer 2009} or \cite{Maurer 2015}, but see also
the next section.

\section{Bounding the Rademacher complexity of vector-valued classes\label%
{Section Bounding}}

At first glance the expression 
\begin{equation*}
\mathbb{E}\sup_{f\in \tciFourier }\sum_{i,k}\epsilon _{ik}f_{k}\left(
x_{i}\right) 
\end{equation*}%
appears difficult to bound. Nevertheless there are some general techniques
which can be used, such as the reduction to scalar classes, or the use of
duality for linear classes. We also give an example in a truly infinite
dimensional setting.

\subsection{Reduction to component classes}

Suppose $\tciFourier _{1},...,\tciFourier _{K}$ are classes of scalar valued
functions and define a vector-valued class $\prod_{k}\tciFourier _{k}$ with
values in $%
\mathbb{R}
^{K}$ by 
\begin{equation*}
\prod_{k}\tciFourier _{k}=\left\{ x\mapsto \left( f_{1}\left( x\right)
,...,f_{K}\left( x\right) \right) :f_{k}\in \tciFourier _{k}\right\} .
\end{equation*}%
Then, since the constraints are independent, the Rademacher average of the
product class%
\begin{equation}
\mathbb{E}\sup_{f\in \prod_{k}\tciFourier _{k}}\sum_{i,k}\epsilon
_{ik}f_{k}\left( x_{i}\right) =\sum_{k}\mathbb{E}\sup_{f\in \tciFourier
_{k}}\sum_{i}\epsilon _{i}f\left( x_{i}\right) 
\label{product class identity}
\end{equation}%
is just the sum of the Rademacher averages of the scalar valued component
classes. Now let $\tciFourier $ be any function class with values in $%
\mathbb{R}
^{K}$ and for $k\in \left\{ 1,...,K\right\} $ define a scalar-valued class $%
\tciFourier _{k}$ by%
\begin{equation*}
\tciFourier _{k}=\left\{ x\mapsto f_{k}\left( x\right) :f=\left(
f_{1},...,f_{k},...,f_{K}\right) \in \tciFourier \right\} .
\end{equation*}%
Then $\tciFourier \subseteq \prod_{k}\tciFourier _{k}$, so by the identity (%
\ref{product class identity}) 
\begin{equation*}
\mathbb{E}\sup_{f\in \tciFourier }\sum_{i,k}\epsilon _{ik}f_{k}\left(
x_{i}\right) \leq \sum_{k}\mathbb{E}\sup_{f\in \tciFourier
_{k}}\sum_{i}\epsilon _{i}f\left( x_{i}\right) .
\end{equation*}%
This is loose in many interesting cases, but for product classes it is
unimprovable.

\subsection{Linear classes defined by norms}

Let $H$ be a separable real Hilbert-space and let $\mathcal{B}\left( H,%
\mathbb{R}
^{K}\right) $ be the set of bounded linear transformations from $H$ to $%
\mathbb{R}
^{K}$. Then every member of $\mathcal{B}\left( H,%
\mathbb{R}
^{K}\right) $ is characterized by a sequence of weight vectors $\left(
w_{1},...,w_{K}\right) $ with $w_{k}\in H$. Let $\left\Vert \left\Vert
.\right\Vert \right\Vert $ be a norm on $\mathcal{B}\left( H,%
\mathbb{R}
^{K}\right) $ with dual norm $\left\Vert \left\Vert .\right\Vert \right\Vert
_{\ast }$. Fix some real number $B$, and define a class $\tciFourier $ of
functions from $H$ to $%
\mathbb{R}
^{K}$ by 
\begin{equation*}
\tciFourier =\left\{ x\mapsto Wx:W\in \mathcal{B}\left( H,%
\mathbb{R}
^{K}\right) ,\left\Vert \left\Vert W\right\Vert \right\Vert \leq B\right\} .
\end{equation*}%
Then 
\begin{eqnarray*}
\mathbb{E}\sup_{f\in \tciFourier }\sum_{i,k}\epsilon _{ik}f_{k}\left(
x_{i}\right)  &=&\mathbb{E}\sup_{\left\Vert \left\Vert \left(
w_{1},...,w_{K}\right) \right\Vert \right\Vert \leq B}\sum_{k}\left\langle
w_{k},\sum_{i}\epsilon _{ik}x_{i}\right\rangle  \\
&=&\mathbb{E}\sup_{\left\Vert \left\Vert W\right\Vert \right\Vert \leq
B}tr\left( D^{\ast }W\right) \leq B~\mathbb{E}\left\Vert \left\Vert D^{\ast
}\right\Vert \right\Vert _{\ast },
\end{eqnarray*}%
where $D\in \mathcal{B}\left( H,%
\mathbb{R}
^{K}\right) $ is the random transformation%
\begin{equation*}
v\mapsto \left( \left\langle v,\sum_{i}\epsilon _{i1}x_{i}\right\rangle
,...,\left\langle v,\sum_{i}\epsilon _{iK}x_{i}\right\rangle \right) .
\end{equation*}%
The details of bounding $\mathbb{E}\left\Vert \left\Vert D^{\ast
}\right\Vert \right\Vert _{\ast }$ then depend on the nature of the norm $%
\left\Vert \left\Vert .\right\Vert \right\Vert $. The simplest case is the
Hilbert-Schmidt or Frobenius norm, where%
\begin{equation*}
\mathbb{E}\left\Vert \left\Vert D^{\ast }\right\Vert \right\Vert _{\ast }=%
\mathbb{E}\sqrt{\sum_{k}\left\Vert \sum_{i}\epsilon _{ik}x_{i}\right\Vert
^{2}}\leq \sqrt{K\sum_{i}\left\Vert x_{i}\right\Vert ^{2}}.
\end{equation*}%
More interesting are mixed norms or the trace norm. A valuable ressource for
this approach is \cite{KakadeEtAl 2012}. 

\subsection{Operator valued kernels}

We give an example in a truly infinite dimensional setting and refer to the
mechanism of learning vector valued functions as exposed in \cite{Massi 2005}%
. There is a generic separable Hilbert space $H$ and a kernel $\kappa :%
\mathcal{X}\times \mathcal{X\rightarrow L}\left( H\right) $ satisfying
certain positivity and regularity properties as decribed in \cite{Massi 2005}%
, where $\mathcal{X}$ is some arbitrary input space. Then there exists an
induced feature-map $\Phi :\mathcal{X\rightarrow L}\left( \ell _{2},H\right) 
$ such that the kernel is given by 
\begin{equation*}
\kappa \left( x,y\right) =\Phi \left( x\right) \Phi ^{\ast }\left( y\right) 
\end{equation*}%
and the class of $H$-valued functions to be learned is 
\begin{equation*}
\left\{ x\mapsto f_{w}\left( x\right) =\Phi \left( x\right) w:\left\Vert
w\right\Vert \leq B\right\} ,
\end{equation*}%
where also $\left\Vert f_{w}\left( x\right) \right\Vert _{H}=\left\Vert
w\right\Vert \sqrt{\kappa \left( x,x\right) }$. Then for any sample $\mathbf{%
x}=\left( x_{1},...,x_{n}\right) \in \mathcal{X}^{n}$ and $L$-Lipschitz loss
functions $h_{i}:H\rightarrow 
\mathbb{R}
$ we have%
\begin{eqnarray*}
\mathbb{E}\sup_{\left\Vert w\right\Vert \leq 1}\sum_{i}\epsilon
_{i}h_{i}\left( \Phi \left( x_{i}\right) w\right)  &\leq &\sqrt{2}L\mathbb{E}%
\sup_{\left\Vert w\right\Vert \leq 1}\sum_{i,k}\epsilon _{ik}\left\langle
\Phi \left( x_{i}\right) w,e_{k}\right\rangle  \\
&=&\sqrt{2}L\mathbb{E}\sup_{\left\Vert w\right\Vert \leq 1}\left\langle
w,\sum_{i,k}\epsilon _{ik}\Phi \left( x_{i}\right) ^{\ast
}e_{k}\right\rangle  \\
&\leq &\sqrt{2}LB\mathbb{E}\left\Vert \sum_{i,k}\epsilon _{ik}\Phi \left(
x_{i}\right) ^{\ast }e_{k}\right\Vert  \\
&\leq &\sqrt{2}LB\left( \sum_{i,k}\left\Vert \Phi \left( x_{i}\right) ^{\ast
}e_{k}\right\Vert ^{2}\right) ^{1/2} \\
&=&\sqrt{2}LB\left( \sum_{i}tr\text{~}\kappa \left( x_{i},x_{i}\right)
\right) ^{1/2}.
\end{eqnarray*}%
Here we used Corollary \ref{Corollary Rademacher} in the first and
Cauchy-Schwarz in the second inequality. Then we use Jensen's inequality
combined with orthonormality of the Rademacher sequence. For the result to
make sense we need the $\kappa \left( x_{i},x_{i}\right) $ to be trace
class. In the case $H=%
\mathbb{R}
$ we obtain the standard result for the scalar case, as in \cite{Bartlett
2002}. The bound above can be used to prove a non-asymptotic upper bound for
the algorithm described in \cite{Caponnetto 2007}, where vector-valued
regression with square loss and Tychonov regularization in $\left\Vert
f_{w}\right\Vert =\left\Vert w\right\Vert $ is considered.

\section{Proof of the contraction inequality\label{Section Proofs}}

We start with some simple observations on subgaussian random variables.

\begin{lemma}
\label{Lemma Subgaussian}If $X$ is subgaussian with subgaussian-constant $b$
and $v$ is a unit vector in $%
\mathbb{R}
^{K}$ then%
\begin{equation*}
\Pr \left\{ \left\vert \sum_{k=1}^{K}v_{k}X_{k}\right\vert >t\right\} \leq
2e^{-t^{2}/\left( 2b^{2}\right) },
\end{equation*}%
where $X_{1},...,X_{K}$ are independent copies of $X$.
\end{lemma}

\begin{proof}
For any $\lambda \in 
\mathbb{R}
$%
\begin{eqnarray*}
\mathbb{E}\exp \left( \lambda \sum_{k}v_{k}X_{k}\right)  &=&\prod_{k}\mathbb{%
E}\exp \left( \lambda v_{k}X_{k}\right)  \\
&\leq &\prod_{k}\exp \left( \lambda ^{2}\frac{b^{2}}{2}v_{k}^{2}\right)  \\
&=&\exp \left( \frac{\lambda ^{2}b^{2}}{2}\right) \text{.}
\end{eqnarray*}%
The first line follows from independence of the $X_{i}$, the next because $X$
is sub-gaussian, and the last because $v$ is a unit vector. It then follows
from Markov's inequality that%
\begin{eqnarray*}
\Pr \left\{ \sum_{k}v_{k}X_{k}>t\right\}  &\leq &\mathbb{E}\exp \left(
\lambda \left( \sum_{k}v_{k}X_{k}-t\right) \right)  \\
&\leq &\exp \left( \frac{\lambda ^{2}b^{2}}{2}-\lambda t\right)  \\
&=&e^{-t^{2}/\left( 2b^{2}\right) },
\end{eqnarray*}%
where the last identity is obtained by optimizing in $\lambda $. The
conclusion follows from a union bound.\bigskip 
\end{proof}

For the purpose of vector-contraction inequalities the crucial property of
sub-gaussian random variables is the following.

\begin{proposition}
\label{Proposition Bound}Let $X$ be nontrivial and subgaussian with
subgaussian parameter $b$ and let $\mathbf{X=}\left(
X_{1},...,X_{K},...\right) $ be an infinite sequence of independent copies
of $X$. Then

(i) For every $v\in \ell _{2}$ the sequence of random variables $%
Y_{K}=\sum_{i=1}^{K}X_{k}v_{k}$ converges in $L_{p}$ for $0<p<\infty $ to a
random variable denoted by $\sum_{k=1}^{\infty }X_{k}v_{k}$. The map $%
v\mapsto \sum_{k=1}^{\infty }X_{k}v_{k}$ is a bounded linear transformation
from $\ell _{2}$ to $L_{p}$.

(ii) There exists a constant $C<\infty $ such that for every $v\in \ell _{2}$
\begin{equation*}
\left\Vert v\right\Vert \leq C\mathbb{E}\left\vert \sum_{k=1}^{\infty
}X_{k}v_{k}\right\vert .
\end{equation*}
\end{proposition}

The proof, given below, is easy and modeled after the proof of the
Khintchine inequalities in \cite{Ledoux Talagrand 1991}.

For Rademacher variables the best constant is $C=\sqrt{2}$ (\cite{Szarek
1976}, see also inequality (4.3) in \cite{Ledoux Talagrand 1991} or Theorem
5.20 in \cite{Boucheron13}). In the standard normal case the inequality in
(ii) becomes equality with $C=\sqrt{\pi /2}$. This is an easy consequence of
the rotation invariance of isonormal processes.

\begin{proof}[Proof of Proposition \protect\ref{Proposition Bound}]
Let $X$ have subgaussian-constant $b$.

(i) Assume first that $\left\Vert v\right\Vert =1$. For $0<p<\infty $ it
follows fom integration by parts that for any $v\in \ell _{2}$%
\begin{eqnarray*}
\mathbb{E}\left\vert \sum_{k=1}^{K}v_{k}X_{k}\right\vert ^{p}
&=&p\int_{0}^{\infty }t^{p-1}\Pr \left\{ \left\vert
\sum_{k=1}^{K}v_{k}X_{k}\right\vert >t\right\} dt \\
&\leq &2p\int_{0}^{\infty }t^{p-1}e^{-t^{2}/\left( 2b^{2}\right) }dt,
\end{eqnarray*}%
where the last inequality follows from Lemma \ref{Lemma Subgaussian}. The
last integral is finite and depends only on $p$ and $b$. By homogeneity it
follows that for some constant $B$ and any $v\in \ell _{2}$ 
\begin{equation*}
\left( \mathbb{E}\left\vert \sum_{k=1}^{K}v_{k}X_{k}\right\vert ^{p}\right)
^{1/p}\leq B\left( \sum_{k=1}^{K}v_{k}^{2}\right) ^{1/2}
\end{equation*}%
which implies convergence in $L_{p}$. This proves existence and boundedness
of the map $v\mapsto \sum_{k=1}^{\infty }v_{k}X_{k}$. Linearity is
established with standard arguments.

(ii) Let $C$ be the finite constant 
\begin{equation*}
C:=\frac{\left( 8\int_{0}^{\infty }t^{3}e^{-t^{2}/\left( 2b^{2}\right)
}dt\right) ^{1/2}}{\mathbb{E}\left[ X^{2}\right] ^{3/2}}.
\end{equation*}%
It suffices to prove the conclusion for unit vectors $v\in \ell _{2}$. From
the first part we obtain%
\begin{equation*}
\mathbb{E}\left\vert \sum_{k}v_{k}X_{k}\right\vert ^{4}\leq
8\int_{0}^{\infty }t^{3}e^{-t^{2}/\left( 2b^{2}\right) }dt
\end{equation*}%
Combined with Hoelder's inequality this implies 
\begin{eqnarray*}
\mathbb{E}\left[ X^{2}\right] &=&\mathbb{E}\left( \left\vert
\sum_{k}v_{k}X_{k}\right\vert ^{2}\right) =\mathbb{E}\left( \left\vert
\sum_{k}v_{k}X_{k}\right\vert ^{4/3}\left\vert \sum_{k}v_{k}X_{k}\right\vert
^{2/3}\right) \\
&\leq &\left( \mathbb{E}\left\vert \sum_{k}v_{k}X_{k}\right\vert ^{4}\right)
^{1/3}\left( \mathbb{E}\left\vert \sum_{k}v_{k}X_{k}\right\vert \right)
^{2/3} \\
&\leq &\left( 8\int_{0}^{\infty }t^{3}e^{-t^{2}/\left( 2b^{2}\right)
}dt\right) ^{1/3}\left( \mathbb{E}\left\vert \sum_{k}v_{k}X_{k}\right\vert
\right) ^{2/3}.
\end{eqnarray*}%
Dividing by $\mathbb{E}\left[ X^{2}\right] $ and taking the power of $3/2$
gives 
\begin{equation*}
1\leq C\mathbb{E}\left\vert \sum_{k}v_{k}X_{k}\right\vert \text{.}
\end{equation*}%
\bigskip
\end{proof}

To prove the main vector contraction result we first consider only a single
Rademacher variable $\epsilon $ and then complete the proof by
induction.\bigskip

\begin{lemma}
\label{Lemma Key}Let $X$ be nontrivial, symmetric and subgaussian. Then
there exists a constant $C<\infty $ such that for any countable set $%
\mathcal{S}$ and functions $\psi :\mathcal{S}\rightarrow 
\mathbb{R}
$, $\phi :\mathcal{S}\rightarrow \ell _{2}$ and $f:\mathcal{S\rightarrow 
\mathbb{R}
}$ satisfying%
\begin{equation*}
\forall s,s^{\prime }\in \mathcal{S}\text{, }\psi \left( s\right) -\psi
\left( s^{\prime }\right) \leq \left\Vert \phi \left( s\right) -\phi \left(
s^{\prime }\right) \right\Vert
\end{equation*}%
we have.%
\begin{equation*}
\mathbb{E}\sup_{s\in \mathcal{S}}\epsilon \psi \left( s\right) +f\left(
s\right) \leq C\mathbb{E}\sup_{s\in \mathcal{S}}\sum_{k}X_{k}\phi \left(
s\right) _{k}+f\left( s\right) ,
\end{equation*}%
where the $X_{k}$ are independent copies of $X$ for $1\leq k\leq \infty $,
and $\phi \left( s\right) _{k}$ is the $k$-th coordinate of $\phi \left(
s\right) $.
\end{lemma}

\begin{proof}
For $C$ we take the constant of Proposition \ref{Proposition Bound} and we
let $Y=CX$ and $Y_{k}=CX_{k}$ so that for every $v\in \ell _{2}$ 
\begin{equation}
\left\Vert v\right\Vert \leq \mathbb{E}\left\vert
\sum_{k}v_{k}Y_{k}\right\vert .  \label{eq proposition}
\end{equation}%
Let $\delta >0$ be arbitrary. Then, by definition of the Rademacher variable,%
\begin{align}
& 2\mathbb{E}\sup_{s\in \mathcal{S}}\left( \epsilon \psi \left( s\right)
+f\left( s\right) \right) -\delta  \notag \\
& =\sup_{s_{1},s_{2}\in \mathcal{S}}\psi \left( s_{1}\right) +f\left(
s_{1}\right) -\psi \left( s_{2}\right) +f\left( s_{2}\right) -\delta  \notag
\\
& \leq \psi \left( s_{1}^{\ast }\right) -\psi \left( s_{2}^{\ast }\right)
+f\left( s_{1}^{\ast }\right) +f\left( s_{2}^{\ast }\right)
\label{eq maximizer} \\
& \leq \left\Vert \phi \left( s_{1}^{\ast }\right) -\phi \left( s_{2}^{\ast
}\right) \right\Vert +f\left( s_{1}^{\ast }\right) +f\left( s_{2}^{\ast
}\right)  \label{eq lipschitz} \\
& \leq \mathbb{E}\left\vert \sum_{k}Y_{k}\left( \phi \left( s_{1}^{\ast
}\right) _{k}-\phi \left( s_{2}^{\ast }\right) _{k}\right) \right\vert
+f\left( s_{1}^{\ast }\right) +f\left( s_{2}^{\ast }\right)  \label{eq lemma}
\\
& \leq \mathbb{E}\sup_{s_{1},s_{2}\in \mathcal{S}}\left\vert
\sum_{k}Y_{k}\phi \left( s_{1}\right) _{k}-\sum_{k}Y_{k}\phi \left(
s_{2}\right) _{k}\right\vert +f\left( s_{1}\right) +f\left( s_{2}\right)
\label{eq bound supremum} \\
& =\mathbb{E}\sup_{s_{1}\in \mathcal{S}}\sum_{k}Y_{k}\phi \left(
s_{1}\right) _{k}+f\left( s_{1}\right) +\mathbb{E}\sup_{s_{2}\in \mathcal{S}%
}-\sum_{k}Y_{k}\phi \left( s_{2}\right) _{k}+f\left( s_{2}\right)
\label{eq drop abs} \\
& =2\left( \mathbb{E}\sup_{s\in \mathcal{S}}\sum_{k}Y_{k}\phi \left(
s\right) _{k}+f\left( s\right) \right) .  \label{eq symmetry}
\end{align}%
In (\ref{eq maximizer}) we pass to approximate maximizers $s_{1}^{\ast
},s_{2}^{\ast }\in \mathcal{S}$, in (\ref{eq lipschitz}) we use the assumed
Lipschitz property relating $\psi $ and $\phi $, and in (\ref{eq lemma}) we
apply inequality (\ref{eq proposition}). In (\ref{eq bound supremum}) we use
linearity and bound by a supremum in $s_{1}$ and $s_{2}$. In this expression
we can simply drop the absolute value, because for any fixed configuration
of the $Y_{k}$ the maximum will be attained when the difference is positive,
since the remaining expression $f\left( s_{1}\right) +f\left( s_{2}\right) $
is invariant under the exchange of $s_{1}$ and $s_{2}$. This gives (\ref{eq
drop abs}). The identity (\ref{eq symmetry}) then follows from the symmetry
of the variables $Y_{k}$. Since $\delta >0$ was arbitrary, the result
follows.\bigskip
\end{proof}

\begin{proof}[Proof of Theorem \protect\ref{Theorem Main}]
The constant $C$ and the $Y_{k}$ are chosen as in the previous Lemma. We
prove by induction that $\forall $ $m\in \left\{ 0,...,n\right\} $%
\begin{equation*}
\mathbb{E}\sup_{s\in \mathcal{S}}\sum_{i}\epsilon _{i}\psi _{i}\left(
s\right) \leq \mathbb{E}\left[ \sup_{s\in \mathcal{S}}\sum_{i:1\leq i\leq
m}\sum_{k}Y_{ik}\phi _{i}\left( s\right) _{k}+\sum_{i:m<i\leq n}^{n}\epsilon
_{i}\psi _{i}\left( s\right) \right] .
\end{equation*}%
The result then follows for $m=n$. The case $m=0$ is an obvious identity.
Assume the claim to hold for fixed $m-1$ , with $m\leq n$. We denote $%
\mathbb{E}_{m}$ $=\mathbb{E}\left[ .|\left\{ \epsilon _{i},Y_{ik}:i\neq
m\right\} \right] $ and define $f:\mathcal{S\rightarrow 
\mathbb{R}
}$ by%
\begin{equation*}
f\left( s\right) =\sum_{i:1\leq i<m}\sum_{k}Y_{ik}\phi _{i}\left( s\right)
_{k}+\sum_{i:m<i\leq n}^{n}\epsilon _{i}\psi _{i}\left( s\right) .
\end{equation*}%
Then%
\begin{eqnarray*}
\mathbb{E}\sup_{s\in \mathcal{S}}\sum_{i}\sigma _{i}\psi _{i}\left( \mathbf{c%
}\right) &\leq &\mathbb{E}\left[ \sup_{s\in \mathcal{S}}\sum_{i:1\leq
i<m}\sum_{k}Y_{ik}\phi _{i}\left( s\right) _{k}+\sum_{i:m\leq i\leq
n}^{n}\epsilon _{i}\psi _{i}\left( s\right) \right] \\
&=&\mathbb{E~E}_{m}\sup_{s\in \mathcal{S}}\left( \epsilon _{m}\psi
_{m}\left( s\right) +f\left( s\right) \right) \\
&\leq &\mathbb{E~E}_{m}\sup_{s\in \mathcal{S}}\sum_{k}Y_{mk}\phi _{m}\left(
s\right) _{k}+f\left( s\right) \\
&=&\mathbb{E}~\sup_{s\in \mathcal{S}}\sum_{i:1\leq i\leq
m}\sum_{k}Y_{ik}\phi _{i}\left( s\right) _{k}+\sum_{i:m<i\leq n}\epsilon
_{i}\psi _{i}\left( s\right) .
\end{eqnarray*}%
The first inequality is the induction hypothesis, the second is Lemma \ref%
{Lemma Key}.\bigskip
\end{proof}

\section{A negative result\label{Section negative result}}

Conjecture \ref{Conjecture False} can be refuted by a simple counterexample.
Let $\mathcal{X}=\ell _{2}$ with canonical basis $\left( e_{i}\right) $ and
set $x_{i}=e_{i}$ for $1\leq i\leq n$. Let $\tciFourier $ be the unit ball
in the set of bounded operators $\mathcal{B}\left( \ell _{2}\right) $, and
for $h$ we take the function $h:x\in \ell _{2}\mapsto \left\Vert
x\right\Vert $, which has Lipschitz norm equal to one. 

If the conjecture was true then there is a universal constant $K$ such that%
\begin{equation}
\mathbb{E}\sup_{T\in \mathcal{B}\left( H\right) :\left\Vert T\right\Vert
_{\infty }\leq 1}\sum_{i}\epsilon _{i}\left\Vert Tx_{i}\right\Vert \leq K%
\mathbb{E}\sup_{T\in \mathcal{B}\left( H\right) :\left\Vert T\right\Vert
_{\infty }\leq 1}\left\Vert \sum_{i}\epsilon _{i}Tx_{i}\right\Vert .
\label{Absurd inequality}
\end{equation}%
For any Rademacher sequence $\epsilon =\left( \epsilon _{i}\right) $ we let $%
T_{\epsilon }$ be the operator defined by $T_{\epsilon }e_{i}=e_{i}$ if $%
i\leq n$ and $\epsilon _{i}=1$, and $T_{\epsilon }=0$ in all other cases.
Clearly $T_{\epsilon }$ has norm $\left\Vert T_{\epsilon }\right\Vert
_{\infty }\leq 1$ (it is the orthogonal projection to the subspace spanned
by the basis vectors $e_{i}$ such that $\epsilon _{i}=1$). Then%
\begin{equation*}
\frac{n}{2}=\mathbb{E}\left\vert \left\{ i:\epsilon _{i}=1\right\}
\right\vert =\mathbb{E}\sum_{i}\epsilon _{i}\left\Vert T_{\epsilon
}x_{i}\right\Vert \leq \mathbb{E}\sup_{T\in \mathcal{B}\left( H\right)
:\left\Vert T\right\Vert _{\infty }\leq 1}\sum_{i}\epsilon _{i}\left\Vert
Tx_{i}\right\Vert .
\end{equation*}%
But on the other hand, the orthonormality of the Rademacher sequence implies
that%
\begin{equation*}
\mathbb{E}\sup_{T\in \mathcal{B}\left( H\right) :\left\Vert T\right\Vert
_{\infty }\leq 1}\left\Vert \sum_{i}\epsilon _{i}Tx_{i}\right\Vert \leq 
\mathbb{E}\left\Vert \sum_{i}\epsilon _{i}e_{i}\right\Vert \leq \sqrt{n}.
\end{equation*}%
With (\ref{Absurd inequality}) we obtain $n/2\leq K\sqrt{n}$ for some
universal constant $K$, which is absurd.

\end{document}